\documentclass[conference]{IEEEtran}
\IEEEoverridecommandlockouts
\usepackage{cite}
\usepackage{amsmath,amssymb,amsfonts}
\usepackage[noend]{algpseudocode}
\usepackage{algorithm}
\usepackage{graphicx}
\usepackage{textcomp}
\usepackage[dvipsnames]{xcolor}
\usepackage{xspace}
\usepackage{tabularx} 
\usepackage{svg}
\usepackage{amsthm}
\usepackage{soul} 
\usepackage{complexity}
\usepackage{tikz}
\usepackage{siunitx} 
\usepackage{etoolbox}
\usepackage{pgfplots}

\usetikzlibrary{tikzmark}
\pgfplotsset{compat=1.3}
\definecolor{series_runningtime}{rgb}{0.55,0.05,0.05}   
\definecolor{series_precisioninbits}{rgb}{0.05,0.05,0.55}
\usepackage[hidelinks]{hyperref}

\theoremstyle{definition}
\newtheorem{definition}{Definition}
\theoremstyle{lemma}
\newtheorem{lemma}{Lemma}
\newtheorem{theorem}{Theorem}

\makeatletter
\renewcommand{\ALG@beginalgorithmic}{\small}
\makeatother

\ifdefined\iros
\newcommand\shortenXor[2]{#2}
\newcommand{\gobble}[1]{}
\newcommand{\gobblexor}[2]{#2}  
\else
\newcommand\shortenXor[2]{#1}
\newcommand{\gobble}[1]{#1}
\newcommand{\gobblexor}[2]{#1} 
\fi

\renewcommand{\emptyset}{\varnothing}

\DeclareMathOperator{\Walks}{\mathrm{Walks}}

\newcommand{\src}{{\rm src}}
\newcommand{\srcfunc}[1]{\src({#1})}
\newcommand{\tgt}{{\rm tgt}}
\newcommand{\tgtfunc}[1]{\tgt({#1})}
\newcommand{\pow}[1]{\ensuremath{\raisebox{.15\baselineskip}{\Large\ensuremath{\wp}}({#1})}\xspace}

\newcommand\together{\raisebox{1.5pt}{\ensuremath{{}_\boxminus}}}
\newcommand\apart{\raisebox{1.5pt}{\ensuremath{{}_\boxtimes}}}
\newcommand\senseUnion[1]{\ensuremath{\mathbf{Y}({#1})}}

\newcounter{tecounter}
\newcommand*{\probleminternal}[4]{
	\par
	\medskip
	\noindent\fbox{\parbox{0.98\columnwidth}{
			\textbf{#4: #1} \\[0.05in]
			\renewcommand{\tabcolsep}{2pt}
			\begin{tabularx}{\linewidth}{rX}
				\emph{Input:} & #2 \\
				\emph{Output:} & #3
			\end{tabularx}
		}}
		\par
		\medskip
		\par
	}

\newcommand*{\decproblem}[3]{\probleminternal{#1}{#2}{#3}{Decision Problem}}
\newcommand{\slsh}{\slash\hspace{0pt}}

\graphicspath{ {./images/} }

\begin{document}

\title{
\ifdefined\iros
\LARGE
\phantom{Sensor}\\[-4pt]
Sensor~selection~for~fine-grained~behavior~verification~that~respects~privacy
\vspace*{-12pt}
\else
Sensor selection for fine-grained behavior verification that respects privacy (extended version)
\vspace*{-5pt}
\fi
}

\author{Rishi Phatak \and Dylan A. Shell
\thanks{\hspace*{-2.2ex}\scriptsize Both authors are affiliated with Dept. of Computer Science \& Engineering, 
Texas A\&M University, College Station, TX, USA.
        \scalebox{0.8}{\textsf{\{rishi.phatak$\,|\,$dshell\}@tamu.edu}}.
        }
}


\maketitle

\vspace*{-15pt}
\begin{abstract}
A useful capability is that of classifying some agent's behavior using
data from a sequence, or trace, of sensor measurements. The sensor selection
problem involves choosing a subset of available sensors to ensure that,
when generated, observation traces will contain enough information to determine whether the agent's
activities match some pattern. 
In generalizing prior work, this paper studies a formulation in which
multiple behavioral itineraries may be supplied, with sensors selected to
distinguish between behaviors. This allows one to pose fine-grained questions,
e.g., to position the agent's activity on a spectrum.
In addition, with multiple itineraries, one can 
also ask about choices of sensors where some behavior is always 
plausibly concealed by (or mistaken for) another.
%
Using sensor ambiguity to limit the acquisition of knowledge is a strong  privacy guarantee, a form of guarantee which some earlier work examined under formulations
distinct from our inter-itinerary conflation approach.
By concretely
formulating privacy requirements for sensor selection, 
this paper
connects both lines of work in a novel fashion: privacy---where there is a bound from above, and
behavior verification---where sensors choices are bounded from below.  
We examine the
worst-case computational complexity that results from both types of bounds,
proving that upper bounds are 
more challenging 
under standard computational
complexity assumptions.  The problem is intractable in general, but we introduce an
approach to solving this problem that can exploit interrelationships
between constraints, and identify opportunities for optimizations.  Case studies
are presented to demonstrate the usefulness and scalability of our proposed
solution, and to assess the impact of the optimizations.
\end{abstract}


\section{Introduction}
\label{introduction-section}

The problems of activity recognition~\cite{yu2011story},
surveillance~\cite{yu2010cyber,uddin2018ambient,ramos2021daily}, suspicious
and/or anomalous behavior detection~\cite{rowe2005detecting}, fault
diagnosis~\cite{sampath1995diagnosability,cassez2008fault}, and task
monitoring~\cite{wang18}\,---despite applying to distinct scenarios---\,all
involve the challenge of  analyzing behavior on the basis of streams of
observations from sensors.  Sensor selection and activation problems (as
studied by~\cite{rahmani2021sensor,cassez2008fault,yin2017minimization,
wang2018optimizing}) are concerned with selecting a set of sensors to provide
\emph{sufficient} information to reach conclusions that are both unequivocal
and correct.
Yet, \emph{too much} information may be
detrimental\,---\,for instance, in elder care and independent living applications
(cf.~\cite{uddin2018ambient}), capturing or divulging sensitive\slsh
inappropriate information could calamitous enough to
be considered a showstopper.

\newcommand{\region}[1]{\scalebox{0.8}{{\textsf{\textsc{#1}}}}\xspace}
\newcommand{\pool}{\region{pool}}
\newcommand{\study}{\region{study}}
\newcommand{\bedroom}{\region{bedroom}}
\newcommand{\bathroom}{\region{bathroom}}
\newcommand{\kitchen}{\region{kitchen}}
\newcommand{\ld}{\region{lounge/dining}}
\newcommand{\fyard}{\region{front yard}}
\newcommand{\byard}{\region{backyard}}
\newcommand{\person}{Myra\xspace}

\begin{figure}[t!]
\vspace*{-4pt}
\begin{minipage}{0.59\linewidth}
  \includegraphics[trim=0pt 0pt 690pt 0pt, clip,width=\textwidth]{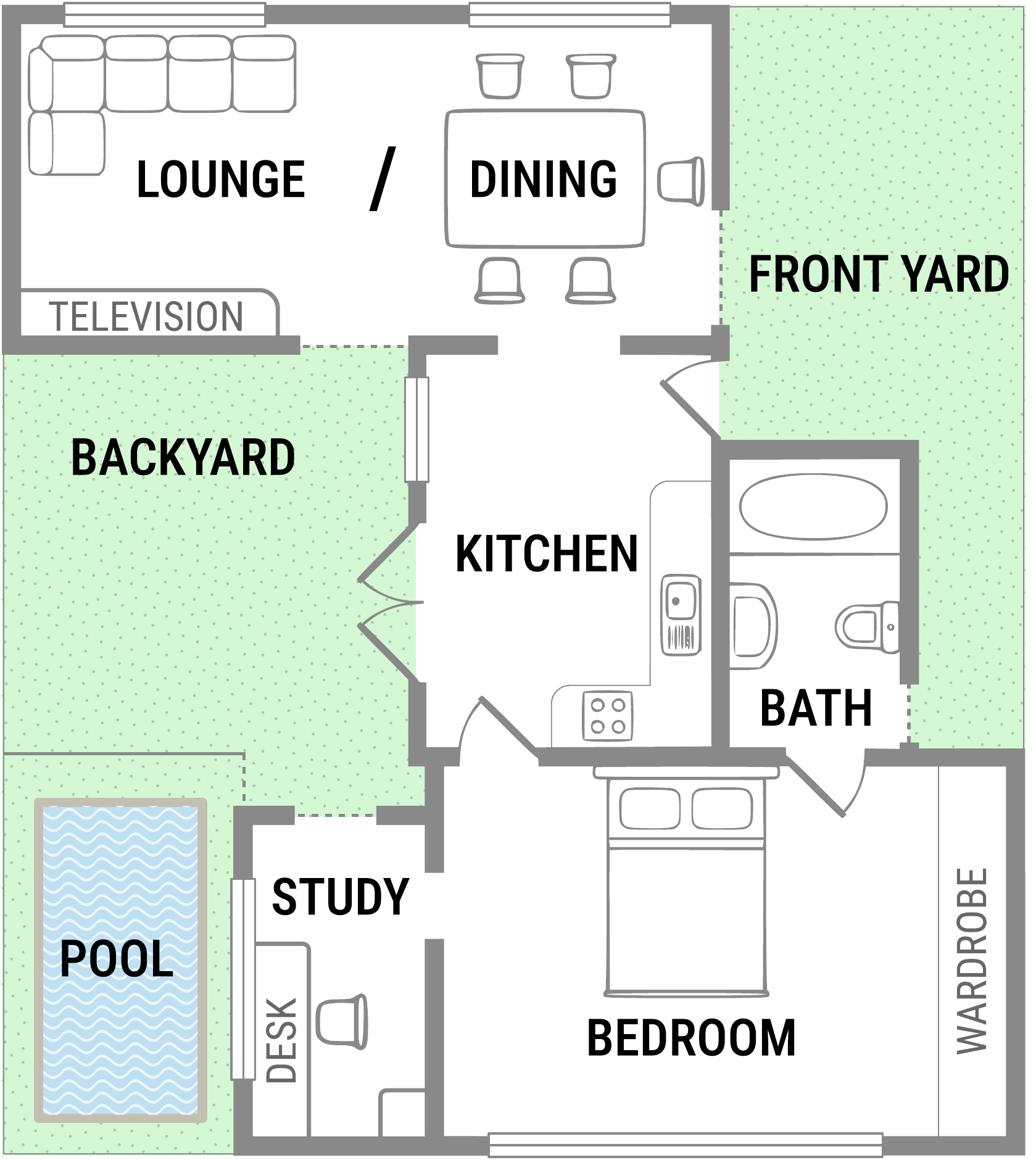}
\end{minipage}
\hfill
\begin{minipage}{0.395\linewidth}
  \caption{\person's assistive living space wherein occupancy detectors can be employed within
            contiguous areas, 
            corresponding here to eight regions: the \pool, \study, \bedroom, \bathroom, \kitchen, \ld, \byard,~and \fyard.
            Different subsets of detectors result in quite distinct sensing granularities: inadequate sensors will mean that the system is incapable of obtaining information needed about the state of the environment,
            too many sensors are invasive and cause privacy concerns.
  }
  \label{PlanWorld}
\end{minipage}
\vspace*{-9pt}
\end{figure}

As a concrete motivating example, consider the house shown in Figure~\ref{PlanWorld}.
Suppose that it is to be turned, via automation, into a `smart home' to serve
as an assisted living space for an elderly person named \person.
Assume that occupancy sensors triggered by physical presence can be
placed in each labelled, contiguous area. 
We might program a system that uses such sensors to track important properties
related to \person's wellness and health goals so that a 
carer can be notified if something is amiss.  
For instance, suppose that to
help fend off dementia, \person does a post-lunch crossword in her study.
To
determine that \person has moved through the house and ended up in the study
doing her crossword, a single occupancy sensor, \study, suffices. 
Unfortunately, when the pool has just been cleaned, 
the chlorine negatively affects \person's sinuses.
To ensure
that she ends up in the study \emph{and} never visits the swimming pool, we
need 2 sensors (\study, \pool).  The increase makes intuitive sense: we are, after all,
now asking for more information about the activity than before.  Notice the 3
kinds of behavior that we can now discriminate between: ones that are both
safe and desirable (never visiting the pool and ending in the study), ones that
are safe but undesirable (never visiting the pool, but also not ending in the
study), and ones that are not safe (visiting the chlorinated pool).

Dinner time is next. We wish to have enough sensing power to tell that \person
has ended up the lounge\slsh dining area, having spent some time in the
kitchen. A pair of sensors (\kitchen, \ld) will do; and to include the study and pool,
these are in addition to the previous 2, giving 4 in total.  But alas, now
\person is annoyed: very occasionally, she enjoys a perfectly
innocent midnight snack and she feels that any sensor that discloses when she has
raided the fridge (and even the frequency of such forays!) is too
invasive.\footnotemark~  She requires that we guarantee that those evenings in
which her bedroom is occupied continuously shall appear identical to those in
which one (or more)
incursions have been made into the kitchen. 
\footnotetext{Her concern is
not misplaced, given the increasing number of attacks on cloud services in
recent years~\cite{chou2013security} from which stored data may be leaked.}

Her request, along with the previous requirements, can be met with 5 sensors (\ld,
\study, \byard, \fyard, \pool). 
Though simplistic, this example illustrates an important idea\,---\,it is not 
enough to reduce the number of sensors to increase privacy, but that sometimes
it may be necessary to activate a different and higher cardinality combination
of sensors to protect sensitive information.

The present paper re-visits the sensor selection model introduced in the IROS'21 paper of 
Rahmani et al.~\cite{rahmani2021sensor},  
advancing and elaborating upon it in order to treat the sort of
problem just described.
In that paper, the authors consider the setting where a
claimant asserts that (future) movements within an
environment will adhere to a given itinerary. Then the objective is to select,
from some set of sensors at specific locations,
a small subset that will detect any deviations from this claim.  
One of present paper's key advances is the ability
to constrain the information obtained from sensors, in
order to meet privacy and non-disclosure requirements.
Further, 
the present paper generalizes the problem so that multiple itineraries are
considered and, consequently, the
objective becomes rather more subtle. In the prior work, the problem is
to select sensors that single out the claimed itinerary from all other
activity; now, when  
closely-related itineraries are provided, the sensors selected must have
adequate resolving power distinguish fine-grain differentiations (recall the 3 kinds of behavior above).

This paper establishes the computational hardness of sensor selection and
optimization under this richer setting (see Section~\ref{hardness}) giving a nuanced 
description of its relation to the constraints introduced to modulate the collected information. Then,
although the problem  is worst-case intractable in general, 
we introduce an exact method
in Section~\ref{algorithm-description}
which treats the sensor selection problem using automata theoretic tools (an approach quite
distinct from the ILP of~\cite{rahmani2021sensor}).  
Multiple itineraries are provided as input and their
interrelationships express constraints\,---\,we examine opportunities to exploit aspects of this structure, which leads 
us to propose some optimizations.  
The empirical results we present in Section~\ref{experimental-results} show that the improvements obtained from the
optimizations are significant, and 
demonstrate how they help improve the scalability of our proposed solution.



\shortenXor{}{Some detail has necessarily been omitted,
the reader is encouraged to refer to \cite{phatak23sensorfull} for the full
authoritative version of the paper.}

\section{Related works}
\label{related-works}

So far, no single model for robotic privacy has yet emerged. 
A useful taxonomy dealing with privacy for robots (and associated intelligent
systems) appears in~\cite{rueben2017taxonomy}. 
Perhaps most visible candidate is
that of differential privacy, used by such
works as~\cite{prorok17privacy, cortes2016differential}. 
There, the underlying formulation builds upon a notion of nearness (originally
with a static database of multiple records), and 
is a less natural fit to treat
the problem of altering the processes by
which data are acquired.  The present work tackles how
privacy (of
even a single entity) may be preserved without any need for addition of noise if
they can exert some degree of control on the tools used to collected that data. 

The idea of obscuring or concealing information is another candidate and is prevalent in the control
community's notion of opacity: an excellent overview for Discrete Event Systems
(DES) is by Jacob, Lesage, and Faure~\cite{jacob2016overview}.
A DES is said to be opaque if a secret has some level of indistinguishability,
a concept very close to the conflation constraints we
define in Section~\ref{problem-statement}. For further reading in the role of
opacity in DES, the reader is referred to \cite{lin2011opacity},
\cite{zaytoon2013overview} and \cite{lafortune2018history}.

Previous work by Masopust and Yin affirms that the properties of detectability and opacity are worst case intractable in general\cite{masopust2019complexity}.
In particular, Cassez et. al.\cite{cassez2012synthesis} showed that determining the opacity of static and dynamic masks was \PSPACE-Complete via formulation of so-called `state-based' and `trace-based' opacities.
In our work, importantly, simply obfuscating states is not enough, as how that
particular state was reached also plays a role.  A second factor
which differentiates our work is that we allow specifications of constraints
between two specified behaviors, instead of making them binary, one-versus-all
decisions.
An important subtlety,
moreover, is that the conflation constraints are directed (cf., also~\cite{okane15discreet}), implying that a more fine grained designation of obfuscation is allowed without necessarily running in both directions.
Thus, we find it more suitable to reduce directly from the inclusion problem than universality.

\section{Problem statement and definitions}
\label{problem-statement}

The environment in which some agent of interest moves is modelled as a discrete structure called the \emph{world graph}:
\begin{definition} [World Graph~\cite{rahmani2021sensor}]

A world graph is an edge-labelled, directed multigraph $\mathcal{G} = (V, E, \src, \tgt, v_0, S,\mathbb{Y},\lambda)$:
\shortenXor{
\begin{itemize}
\item $V$ is a non-empty vertex set,
\item $E$ is a set of edges,
\item $\src : E \to V$ and $\tgt : E \to V$ are source and target
functions, respectively, identifying a source vertex
and target vertex for each edge,
\item $v_0 \in V$ is an initial vertex,
\item $S = \{s_1, s_2,\dots , s_k\}$ is a nonempty finite set of sensors,
\item $\mathbb{Y} = \{Y_{s_1}, Y_{s_2},\dots , Y_{s_k}\}$ is a collection of mutually
disjoint event sets associated to each sensor, and
\item $\lambda : E \to \pow{Y_{s_1} \cup Y_{s_2} \cup\dots\cup Y_{s_k}}$ is a labelling function, which assigns to each edge, a world-observation a set of events.
\end{itemize}
(Here $\pow{X}$, the powerset, denotes all the subsets of $X$.)
}{%
$V$ is a non-empty vertex set;
$E$ is a set of edges;
$\src : E \to V$ and $\tgt : E \to V$ are source and target
functions, respectively, identifying a source vertex
and target vertex for each edge;
$v_0 \in V$ is an initial vertex;
$S = \{s_1, s_2,\dots , s_k\}$ is a nonempty finite set of sensors;
$\mathbb{Y} = \{Y_{s_1}, Y_{s_2},\dots , Y_{s_k}\}$ is a collection of mutually
disjoint event sets associated to each sensor; 
$\lambda : E \to \pow{Y_{s_1} \cup Y_{s_2} \cup\dots\cup Y_{s_k}}$ is a labelling function, which assigns to each edge, a world-observation a set of events.
(Here powerset $\pow{X}$ denotes all the subsets of $X$.)}

\end{definition}

The usefulness of the world graph is that it governs two major aspects of the agent's locomotion: how it may move, and
what would happen if it moved in a certain way.
The agent is known to start its movements at $v_0$ and take connected edges. 

However, the agent cannot make any transitions that are not permitted by the world graph. \person, for example, cannot jump from the \bedroom to the \ld without first going through the \kitchen.
Thus, the collection of all paths that can physically be taken by the agent is defined as follows:

\begin{definition} [Walks~\cite{rahmani2021sensor}]

A string $e_1e_2\dots e_n \in E^*$ is a walk on the world graph if and only if $\srcfunc{e_1} = v_0$ and for all $i \in \{1, \dots, n-1\}$ we have that $\tgtfunc{e_i} = \srcfunc{e_{i+1}}$. The set of all walks over $\mathcal{G}$ is denoted $\Walks(\mathcal{G})$

\end{definition}

Next, we seek to understand what role the sensors play when an agent interacts with the world.
Whenever an edge is crossed, it causes a `sensor response' described by the
label on that edge: those sensors which are associated with the sensor values
in the label (and are turned on\slsh selected) will emit those values.  Returning to the home
in Figure~\ref{PlanWorld}, assume there are sensors in the \bedroom and \study
which measure occupancy.  Then, when \person starts in the bedroom and moves to
the study, we would obtain the event $\{ \bedroom^-, \study^+\}$ for the
transition, with the plus superscript representing an event triggered by detection,
and minus the inverse. The model also allows sensors other
than those which detect occupancy (e.g., non-directed traversals via break beams),
see~\cite{rahmani2021sensor} too. 

To understand the sensor values generated when crossing a single edge where sensors may be turned off, we use a sensor labelling function:

\begin{definition}[Sensor labelling function]
    Let  $\mathcal{G} = (V, E, \src, \tgt, v_0, S,\mathbb{Y},\lambda)$ be a world graph,
    and $M \subseteq S$ a sensor selection from it. 
    For selection $M$,
    the set of all events that could be produced by 
    those sensors 
    will be denoted $\senseUnion{M} = \bigcup_{s \in M} Y_s$.   
    Then the \emph{sensor labelling function} is 
     for all $e \in E$: \vspace*{-10pt}
    \[\lambda_M(e) = 
   \begin{cases}
    \lambda(e) \cap \senseUnion{M} & \text{if }\lambda(e) \cap \senseUnion{M} \neq \emptyset,\\
    \epsilon & \text{otherwise.}
    \end{cases} 
     \]
\end{definition}

(Note that $\epsilon$ here is the standard empty symbol.)
Later in the paper, 
Figure~\ref{3x3Graph} forms an example of an environment with
a world graph whose edges bear appropriate sensor labels.

Now, we may formally define the signature function for a walk and a given sensor set as follows:

\begin{definition} [Signature of a walk~\cite{rahmani2021sensor}]

For a world graph $\mathcal{G} = (V, E, \src, \tgt, v_0, S,\mathbb{Y},\lambda)$ we define function $\beta_{\mathcal{G}}: \Walks(\mathcal{G}) \times \pow{S} \to (\pow{\senseUnion{S}}\setminus \{\emptyset\})^*$ such that for each $r = e_1e_2\dots e_n \in \Walks(\mathcal{G})$ and $M \subseteq S$, $\beta_{\mathcal{G}}(r, M) = z_1z_2\dots z_n$ in which for each $i \in \{1, \dots, n\}$, we have that $z_i = \lambda_M(e_i)$.

\end{definition}

The behavior of the agent will be specified with respect to a 
given world graph and these specifications will describe sequences of edges
the agent may decide to take in the world graph. 
Following the convention of~\cite{rahmani2021sensor}, each is called an 
itinerary.
Subsequent definitions will involve the use of multiple itineraries
in order to constrain what information about the agent's behavior the sensors are
allowed to obtain.


\begin{definition} [Itinerary DFA~\cite{rahmani2021sensor}]

An itinerary DFA over a world graph $\mathcal{G} = (V, E, \src, \tgt, v_0, S,\mathbb{Y},\lambda)$ is a DFA $\mathcal{I} = (Q, E, \delta, q_0, F)$ in which $Q$ is a finite set of states; $E$ is the alphabet; $\delta : Q \times E \to Q$ is the transition function; $q_0$ is the initial state; and $F$ is the set of accepting (final) states.

\end{definition}

With the basic elements given, the next four definitions formalize the different classes of constraints we desire a set of sensors to satisfy. 
Conflation constraints allow one type of behavior to `appear' similar to another, 
while discrimination constraints specify that two behaviors must be distinguishable. 

\begin{definition} [Conflation constraint]
\label{conflation-constraint}
    A conflation constraint on a world graph $\mathcal{G}$ is an ordered pair of itineraries $(\mathcal{I}_a, \mathcal{I}_b)^{\together}$.
\end{definition}

\begin{definition} [Discrimination constraint]
\label{discrimination-constraint}
    A discrimination constraint on a world graph $\mathcal{G}$ is an unordered pair of itineraries $[\mathcal{I}_1, \mathcal{I}_2]^{\apart}$.
\end{definition}

Both types will combined within a graph:

\begin{definition} [Discernment designation]
A \emph{discernment designation} is a 
mixed graph $\mathcal{D}=(I, I_D, I_C)$, with vertices $I$ being a collection of itineraries, along with undirected edges $I_D$ which are a set of discrimination constraints, and arcs (directed edges) $I_C$ which are a set of conflation constraints.
\end{definition}

And, finally, we can state what a satisfying selection entails:

\begin{definition}[Satisfying sensor selection]
\label{def:sats}
Given some discernment designation $\mathcal{D}$,
    a sensor set $M \subseteq S$ is a \emph{satisfying sensor selection for $\mathcal{D}=(I, I_D, I_C)$} if and only if both of the following conditions hold:
    \begin{itemize}
        \item For each $[\mathcal{I}_1, \mathcal{I}_2]^{\apart} \in I_D$ we have that there exist no $w_1 \in \Walks(\mathcal{G}) \cap \mathcal{L}(\mathcal{I}_1)$ and $w_2 \in \Walks(\mathcal{G}) \cap \mathcal{L}(\mathcal{I}_2)$ where $\beta_\mathcal{G}(w_1, M) = \beta_\mathcal{G}(w_2, M)$.
        \item For each $(\mathcal{I}_a, \mathcal{I}_b)^{\together} \in I_C$ we have that for every $w \in \Walks(\mathcal{G}) \cap \mathcal{L}(\mathcal{I}_a)$, there exists $c_w \in \Walks(\mathcal{G}) \cap \mathcal{L}(\mathcal{I}_b)$ where $\beta_\mathcal{G}(w, M) = \beta_\mathcal{G}(c_w, M)$.
    \end{itemize}
\end{definition}

In the above definition, the `$\apart$' constraints correspond to 
\emph{discrimination} requirements, while `$\together$' require
\emph{conflation}.
The importance of the set intersections is that the only things that can really happen are 
walks on the world graph. 
When there is a discrimination constraint, there are no walks from the one itinerary 
that can be confused with one from the other itinerary. 
When there is a conflation constraint, any walk from the first itinerary has
at least one from the second that appears identical. 
Conflation models privacy in the following sense: any putative claim that the agent followed one itinerary can be countered by arguing,  just as plausibility on the basis of the sensor readings, that it followed the other itinerary.
While the  discrimination constraint
is symmetric, the second need not be. (Imagine:
$\{\beta_\mathcal{G}(w, M) | w \in \Walks(\mathcal{G}) \cap \mathcal{L}(\mathcal{I}_1)\} = \{a,b,c,d\}$ while
$\{\beta_\mathcal{G}(w', M) | w' \in \Walks(\mathcal{G}) \cap \mathcal{L}(\mathcal{I}_2)\} = \{a,b,c,d,e\}$. Then 
$(\mathcal{I}_1,\mathcal{I}_2)^{\together}$ is possible, while
$(\mathcal{I}_2,\mathcal{I}_1)^{\together}$ is not.)

Now, we are ready to give the central problem of the paper:

\decproblem{ Minimal sensor selection to accommodate a discernment designation in itineraries (MSSADDI)}
{A world graph 
$\mathcal{G}$,
a discernment designation 
$\mathcal{D}$,
and a natural number $k\in\mathbb{N}$.}
{A satisfying sensor selection \mbox{$M \subseteq S$} for $\mathcal{D}$ on $\mathcal{G}$ with $|M| \leq k$, or `\textsc{Infeasible}' if 
none exist.
}

\section{Signature Automata}
\label{signature-automata}

To understand how we may begin solving MSSADDI and what its theoretical complexity is, we introduce the concept of a signature automaton.
Signature automata are produced from the product automata of an itinerary with the world graph:

\begin{definition} [Product automaton~\cite{rahmani2021sensor}]
Let  $\mathcal{G} = (V, E, \src, \tgt, v_0, S,\mathbb{Y},\lambda)$ be a world graph and 
$\mathcal{I} = (Q, E, \delta, q_0, F)$ be an intinerary DFA.
The product $\mathcal{P}_{\mathcal{G}, \mathcal{I}}$ is 
a partial DFA $\mathcal{P}_{\mathcal{G}, \mathcal{I}} = (Q_\mathcal{P}, E, \delta_{\mathcal{P}}, q_0^{\mathcal{P}}, F_{\mathcal{P}})$ with
\shortenXor{
\begin{itemize}
   \item $Q_{\mathcal{P}} = Q \times V$,
   \item $\delta_{\mathcal{P}}: Q_{\mathcal{P}} \times E \rightarrow 
   Q_{\mathcal{P}} \cup \{\perp\}$ is a function such that for each $(q, v) \in Q_{\mathcal{P}}$ and $e \in E$, 
   $\delta_{\mathcal{P}}((q, v), e)$ is defined to be $\perp$ if 
   $\srcfunc{e} \neq v$, otherwise, 
   $\delta_{\mathcal{P}}((q, v), e) = (\delta(q, e), \tgtfunc{e})$,
   \item $q_0^{\mathcal{P}}=(q_0, v_0)$, and
   \item $F_{\mathcal{P}} = F \times V$.
\end{itemize}
}{
$Q_{\mathcal{P}} = Q \times V$;
$\delta_{\mathcal{P}}: Q_{\mathcal{P}} \times E \rightarrow Q_{\mathcal{P}} \cup \{\perp\}$ is a function such that for each $(q, v) \in Q_{\mathcal{P}}$ and $e \in E$; 
   $\delta_{\mathcal{P}}((q, v), e)$ is defined to be $\perp$ if 
   $\srcfunc{e} \neq v$, otherwise, 
   $\delta_{\mathcal{P}}((q, v), e) = (\delta(q, e), \tgtfunc{e})$;
$q_0^{\mathcal{P}}=(q_0, v_0)$, and
$F_{\mathcal{P}} = F \times V$.
}
\end{definition}

The language of this product automaton, as a DFA, is the collection of (finite-length) sequences from $E$ that can be traced starting at 
$q_0^{\mathcal{P}}$, never producing a $\perp$, and which arrive
at some element in $F_{\mathcal{P}}$.
The language recognized is the set of walks that are within the itinerary $\mathcal{I}$, i.e., $\mathcal{L}(\mathcal{P}_{\mathcal{G}, \mathcal{I}}) = \Walks(\mathcal{G}) \cap \mathcal{L}(\mathcal{I})$.


\begin{definition} [Signature automaton]

Let  $\mathcal{G} = (V, E, \src, \tgt, v_0, S,\mathbb{Y},\lambda)$ be a world graph, 
let $M \subseteq S$ be a sensor selection on it, 
$\mathcal{I} = (Q, E, \delta, q_0, F)$ be an itinerary DFA, and
$\mathcal{P}_{\mathcal{G}, \mathcal{I}}$ be their product. A signature automaton $\mathcal{S}_{\mathcal{G}, \mathcal{I}, M} = (Q_\mathcal{P}, \Sigma, \delta_{\mathcal{S}}, q_0^{\mathcal{P}}, F_{\mathcal{P}})$ is a nondeterministic finite automaton with $\epsilon$-moves (NFA-$\epsilon$) with

\begin{itemize}

\item $\Sigma = \left\{ \lambda_M(e) \;|\; e \in E,
 \lambda_M(e) \neq \epsilon \right\}$

\item $\delta_{\mathcal{S}}: Q_{\mathcal{P}} \times \Sigma 
 \cup \{\epsilon\} \rightarrow \pow{Q_{\mathcal{P}}}$ is a function defined for each $(q, v) \in Q_{\mathcal{P}}$ and $\sigma \in \Sigma \cup \{ \epsilon \}$ such that 
\begin{multline*}
    \delta_{\mathcal{S}}\big((q,v), \sigma\big) =  
    \Big\{\delta_{\mathcal{P}}\big((q, v), e\big) \;\Big| e \in E, \\ \delta_{\mathcal{P}}\big((q, v), e\big) \neq \perp, \lambda_M(e) = \sigma\Big\}.
\end{multline*}
\end{itemize}
\end{definition}


The signature automaton produces all the signatures that could result from following a path in the world graph conforming to the given itinerary. Formally, we have the following:

\begin{lemma} \label{signlang}
%

For world graph 
$\mathcal{G} = (V, E, \src, \tgt, v_0, S,\mathbb{Y},\lambda)$, 
sensor selection $M \subseteq S$,
and itinerary $\mathcal{I} = (Q, E, \delta, q_0, F)$, if
their signature automaton is $\mathcal{S}_{\mathcal{G}, \mathcal{I}, M}$, then:
$$\mathcal{L}(\mathcal{S}_{\mathcal{G}, \mathcal{I}, M}) = \left\{\beta_{\mathcal{G}}(w, M) \;|\; w \in \Walks(\mathcal{G}) \cap \mathcal{L}(\mathcal{I})\right\}.$$
    
\end{lemma}

\begin{proof}

For all $w \in \Walks(\mathcal{G}) \cap \mathcal{L}(\mathcal{I})$ there is a unique sequence of states $q_0^\mathcal{P}, \gobble{q_1^\mathcal{P},} \dots , q_n^\mathcal{P}$ in $\mathcal{P}_{\mathcal{G}, \mathcal{I}}$ such that $q_n^\mathcal{P} \in F_\mathcal{P}$. Following that sequence through the signature automaton returns \gobble{the} signature $\beta_{\mathcal{G}}(w, M)$.
Similarly, any string that is accepted by $\mathcal{S}_{\mathcal{G}, \mathcal{I}, M}$ has a sequence of states $q_0^\mathcal{P}, \gobble{q_1^\mathcal{P},} \dots , q_n^\mathcal{P}$ in $\mathcal{S}_{\mathcal{G}, \mathcal{I}, M}$ such that $q_n^\mathcal{P} \in F_\mathcal{P}$. Following those states through $\mathcal{P}_{\mathcal{G}, \mathcal{I}}$ returns the walk conforming to the itinerary which produced it.
\end{proof}

\gobblexor{Note that the manner in which the signature automaton was produced was simply to replace the alphabet $E$ of the product automaton with the alphabet $\Sigma$.}{Note that signature automaton simply replaces the alphabet $E$ of the product automaton with the alphabet $\Sigma$.}
This introduces nondeterminism in the automaton because two outgoing edges from a vertex in the world graph may produce the same (non-empty) sensor values. Moreover, certain transitions may be made on the empty symbol if no sensor values are produced upon taking an edge in the world graph too.

\gobblexor{The usefulness of the preceding structures becomes clearer from the lemmas that follow.}{The preceding is useful owing to the next pair of lemmas.}

\begin{lemma}
Given world graph $\mathcal{G} = (V, E, \src, \tgt, v_0, S,\mathbb{Y},\lambda)$ and 
itinerary DFAs:
$\mathcal{I}^1 = (Q^1, E, \delta^1, q_0^1, F^1)$ and
$\mathcal{I}^2 = (Q^2, E, \delta^2, q_0^2, F^2)$,
a subset  of sensors $M \subseteq S$ is a satisfying sensor selection
for constraint discrimination of
itineraries
$\mathcal{I}^1$ and $\mathcal{I}^2$ 
if and only if $\mathcal{L}(\mathcal{S}_{\mathcal{G}, \mathcal{I}^1, M}) \cap \mathcal{L}(\mathcal{S}_{\mathcal{G}, \mathcal{I}^2, M}) = \emptyset$.
\end{lemma}

\begin{proof}

Assume that $M$ satisfies the constraint $[\mathcal{I}_1, \mathcal{I}_2]^{\apart}$. This implies that there exist no $w_1$ and $w_2$, with $w_1 \in \Walks(\mathcal{G}) \cap \mathcal{L}(\mathcal{I}_1)$ and $w_2 \in \Walks(\mathcal{G}) \cap \mathcal{L}(\mathcal{I}_2)$, where $\beta_\mathcal{G}(w_1, M) = \beta_\mathcal{G}(w_2, M)$. The previous fact along with Lemma \ref{signlang} implies $\mathcal{L}(\mathcal{S}_{\mathcal{G}, \mathcal{I}^1, M}) \cap \mathcal{L}(\mathcal{S}_{\mathcal{G}, \mathcal{I}^2, M}) = \emptyset$. 
The other way: 
if such $w_1$ and $w_2$ can be found, 
then letting $c = \beta_\mathcal{G}(w_1, M) = \beta_\mathcal{G}(w_2, M)$,
we have that 
$\{ c \} \subseteq \mathcal{L}(\mathcal{S}_{\mathcal{G}, \mathcal{I}^1, M}) \cap \mathcal{L}(\mathcal{S}_{\mathcal{G}, \mathcal{I}^2, M})$.
\end{proof}

Notice that if $\mathcal{L}(\mathcal{I}_1) \cap \mathcal{L}(\mathcal{I}_2) \neq
\emptyset$  then any walks $w_1 = w_2$ taken from this intersection must have
 $\beta_\mathcal{G}(w_1, M) = \beta_\mathcal{G}(w_2, M)$. Any two
itineraries with overlapping languages, and whose overlap falls (partly) within the set of walks,  will yield a sensor selection problem that must be infeasible
when these itineraries are given as a discrimination
constraints.

A similar lemma follows for the conflation constraints.

\begin{lemma}
Given world graph $\mathcal{G} = (V, E, \src, \tgt, v_0, S,\mathbb{Y},\lambda)$ and 
itinerary DFAs:
$\mathcal{I}^1 = (Q^1, E, \delta^1, q_0^1, F^1)$ and
$\mathcal{I}^2 = (Q^2, E, \delta^2, q_0^2, F^2)$,
a subset  of sensors $M \subseteq S$ is a satisfying sensor selection
for constraint conflation of
itineraries
$\mathcal{I}^1$ and $\mathcal{I}^2$ 
if and only if $\mathcal{L}(\mathcal{S}_{\mathcal{G}, \mathcal{I}^1, M}) \subseteq \mathcal{L}(\mathcal{S}_{\mathcal{G}, \mathcal{I}^2, M})$.
\end{lemma}

\begin{proof}

Assume that $M$ satisfies the constraint $(\mathcal{I}_1, \mathcal{I}_2)^{\together}$. This implies that every $w \in \Walks(\mathcal{G}) \cap \mathcal{L}(\mathcal{I}_1)$ has a $c_w \in \Walks(\mathcal{G}) \cap \mathcal{L}(\mathcal{I}_2)$ with $\beta_\mathcal{G}(w, M) = \beta_\mathcal{G}(c_w, M)$. The previous fact along with Lemma \ref{signlang} implies $\mathcal{L}(\mathcal{S}_{\mathcal{G}, \mathcal{I}^1, M}) \subseteq \mathcal{L}(\mathcal{S}_{\mathcal{G}, \mathcal{I}^2, M})$. In the opposite direction, if there exists a $w$ for which no $c_w$ can be found, we know that $\mathcal{L}(\mathcal{S}_{\mathcal{G}, \mathcal{I}^1, M}) \not\subseteq \mathcal{L}(\mathcal{S}_{\mathcal{G}, \mathcal{I}^2, M})$ since $\beta_\mathcal{G}(w, M) \in \mathcal{L}(\mathcal{S}_{\mathcal{G}, \mathcal{I}^1, M})$ but $\beta_\mathcal{G}(w, M) \not\in \mathcal{L}(\mathcal{S}_{\mathcal{G}, \mathcal{I}^2, M})$.
\end{proof}

\section{Complexity of MSSADDI}
\label{hardness}

\subsection{Background and preliminaries}

Before we prove the hardness of MSSADDI, we state some known facts from automata and complexity theory. 

\begin{lemma}
[Savitch's Theorem \cite{savitch1970relationships}]
\label{savitch-theorem}
In the context of complexity classes, we have that $\PSPACE = \NPSPACE$.
\end{lemma}

\begin{lemma}[NFA intersection~\cite{hopcroft2006introduction}]
\label{nfa-intersection}
Given two non-deterministic finite automata (NFAs) $\mathcal{A}\gobble{ = (Q_A, \Sigma, \delta_A, q_0^A, F_A)}$ and $\mathcal{B}\gobble{ = (Q_B, \Sigma, \delta_B, q_0^B, F_B)}$, it can be determined in polynomial time if $\mathcal{L}(\mathcal{A}) \cap \mathcal{L}(\mathcal{B}) = \emptyset$.
\end{lemma}

\begin{lemma}
[NFA inclusion\cite{meyer1972equivalence}]
\label{nfa-inclusion}
Given two non-deterministic finite automata (NFAs) $\mathcal{A}\gobble{ = (Q_A, \Sigma, \delta_A, q_0^A, F_A)}$ and $\mathcal{B}\gobble{ = (Q_B, \Sigma, \delta_B, q_0^B, F_B)}$, it is \PSPACE-Complete to determine if $\mathcal{L}(\mathcal{A}) \subseteq \mathcal{L}(\mathcal{B})$.

\end{lemma}

\subsection{Hardness of MSSADDI}

Next, we investigate the hardness of the problem formulated above. 
Since the original original MSSVI problem~\cite{rahmani2021sensor} is
\NP-Complete (it essentially involves a single itinerary and its
complement, one discrimination constraint, and zero conflation constraints), we
naturally expect the problem to be \NP-Hard.  And this is indeed true (though
the direct proof is straightforward and, hence, omitted).
For the full problem, the question is whether the conflation constraints
contribute additional extra complexity. The answer is in the affirmative, under standard computational
complexity assumptions:

\begin{lemma} \label{MSSADDI-PSPACE}
    MSSADDI is in \PSPACE.
\end{lemma}

\shortenXor{
\begin{proof}

\gobblexor{To show that MSSADDI is in \PSPACE, we \gobble{shall} show that it is in \NPSPACE, and through Lemma~\ref{savitch-theorem}, this implies that it is also in \PSPACE.}{We show that MSSADDI is in \NPSPACE, and through Lemma~\ref{savitch-theorem}, this implies that it is also in \PSPACE.}

\gobblexor{Given this fact, assume that we have `guessed' a sensor selection $M \subseteq S$ which, in polynomial space, must be verified to be a satisfying sensor selection. Thus, we must verify that each $[\mathcal{I}^1, \mathcal{I}^2]^{\apart} \in I_D$ and each $(\mathcal{I}^1, \mathcal{I}^2)^{\together} \in I_C$ is satisfied by $M$.}
{Suppose some sensor selection $M \subseteq S$ has been `guessed' and must, in polynomial space, be verified to meet Definition~\ref{def:sats}, i.e., each $[\mathcal{I}^1, \mathcal{I}^2]^{\apart} \in I_D$ and $(\mathcal{I}^1, \mathcal{I}^2)^{\together} \in I_C$ shown to be satisfied by $M$. }

First, to show any $[\mathcal{I}^1, \mathcal{I}^2]^{\apart} \in I_D$ can be checked in polynomial time (and thus also polynomial space):
construct $\mathcal{S}_{\mathcal{G}, \mathcal{I}^1, M}$ and $\mathcal{S}_{\mathcal{G}, \mathcal{I}^2, M}$. This simply replaces the alphabet in the product, which is of size $|V||Q|$). Then, determining whether $\mathcal{L}(\mathcal{S}_{\mathcal{G}, \mathcal{I}^1, M}) \cap \mathcal{L}(\mathcal{S}_{\mathcal{G}, \mathcal{I}^2, M}) = \emptyset$ is simple (cf.~Lemma~\ref{nfa-intersection}). 

\gobble{Thus, the total amount of time taken to check the discrimination constraints can be upper bounded by $O\left(\sum_{[\mathcal{I}^1, \mathcal{I}^2]^{\apart} \in I_D} \poly(|V||Q_{\mathcal{I}^1}| , |V||Q_{\mathcal{I}^2}|)\right)$ which is polynomial in the input size.}

Next, conflation constraints: follow a similar process to construct their signature automata $\mathcal{S}_{\mathcal{G}, \mathcal{I}^1, M}$ and $\mathcal{S}_{\mathcal{G}, \mathcal{I}^2, M}$,
and ascertain whether $\mathcal{L}(\mathcal{S}_{\mathcal{G}, \mathcal{I}^1, M}) \subseteq \mathcal{L}(\mathcal{S}_{\mathcal{G}, \mathcal{I}^2, M})$. By Lemma \ref{nfa-inclusion}, we know this problem is \PSPACE-Complete, thus, it can be determined using only a polynomial amount of space. 
Hence $\text{MSSADDI} \in \NPSPACE \implies \text{MSSADDI} \in \PSPACE$.
\end{proof}
}
{
\begin{proof}
The full detailed proof appears in \cite{phatak23sensorfull}.
\end{proof}
}

\gobblexor{Next, to show that MSSADDI is \PSPACE-Hard}{Next, for showing hardness}, we reduce from the NFA inclusion problem\gobblexor{ in Lemma~\ref{nfa-inclusion}.}{.}
%
%
One can think of this intuitively as showing that conflation constraints, in solving the inclusion problem on signature automata, cover worst-case instances. 


\begin{lemma} \label{MSSADDI-PSPACE-Hard}
    MSSADDI is \PSPACE-Hard
\end{lemma}

\begin{proof}

We reduce from NFA Inclusion, known to be \PSPACE-Complete (Lemma \ref{nfa-inclusion}).
Given an NFA Inclusion Problem instance 
$x = \langle \mathcal{A} = (Q_A, \Sigma, \delta_A, q_0^A, F_A), 
\mathcal{B} = (Q_B, \Sigma, \delta_B, q_0^B, F_B) \rangle$ 
we form an instance of MSSADDI 
$f(x) = \langle \mathcal{G} = (V, E, \src, \tgt, v_0, S,\mathbb{Y},\lambda), \mathcal{D}=(I, I_D, I_C), k \rangle$.

Every state of $\mathcal{A}$ and $\mathcal{B}$ will be assumed to reachable from their respective start states 
(unreachable states do not contribute to the NFA's language, and are easily trimmed).
We construct $\mathcal{G}$ as follows:---

\begin{enumerate}

\item Let the vertex set be $V = \{ v_0 \} \cup Q_A \cup Q_B$ where $v_0$ is a new vertex not in either $Q_A$ or $Q_B$. 

\item Let the edge set be $E = \{ e_A, e_B \} \cup \{ e_1, e_2, \dots, e_n,\allowbreak e_{n+1}, e_{n+2}, \dots, e_{n+m} \}$. Here $e_A$ is an edge that connects $v_0$ to $q_0^A$ and $e_B$ is an edge connecting $v_0$ to $q_0^B$. Assuming there are $n$ transitions in $\mathcal{A}$ of the form $q_j^A \in \delta_A(q_i^A, \sigma)$, we produce an edge $e_k$ for some $1 \leq k \leq n$ which connects $q_i^A$ to $q_j^A$ for every such $\sigma$. Similarly, if there are $m$ transitions in $\mathcal{B}$ of the form $q_j^B \in \delta_B(q_i^B, \sigma)$, we would have an edge $e_{n+k}$ for some $1 \leq k \leq m$ connecting $q_i^B$ to $q_j^B$ for each $\sigma$. The $\src$ and $\tgt$ functions are defined appropriately for all edges.

\item Let sensor set $S = \{ s_1, \gobble{s_2,} \dots, s_{|\Sigma|} \}$ where each sensor produces exactly one event so that if $\Sigma = \{ \sigma_1, \gobble{\sigma_2,} \dots, \sigma_{|\Sigma|}\}$ then $Y_{s_i} = \{ \sigma_i \}$ and $\mathbb{Y} = \{ Y_{s_1}, \gobble{Y_{s_2},} \dots, Y_{s_{|\Sigma|}} \}$.

\item The edge labelling function is defined as follows. First, let $\lambda(e_A) = \lambda(e_B) = \emptyset$. Then, for each transition in $\mathcal{A}$ of the form $q_j^A \in \delta_A(q_i^A, \sigma)$, if $\sigma = \epsilon$, label that edge with $\emptyset$, otherwise label it with the singleton set $\{ \sigma \}$ for all such $\sigma$. Follow the same procedure again for $\mathcal{B}$.
Note that, by construction, a single sensor may cover an edge from both $\mathcal{A}$ and $\mathcal{B}$. 
This is natural as the given NFAs share the alphabet $\Sigma$.
Importantly: this does not violate the assumption that sensors have pairwise distinct readings. 
Turning some sensor on, means we receive its readings from both regions---that constructed from $\mathcal{A}$ \emph{and} $\mathcal{B}$---or, when turned off, from neither.
\end{enumerate}

The following define $\mathcal{D}$, the discernment designation:---

\begin{enumerate}

\item In the world graph $\mathcal{G}$ constructed in the previous step, let there be $p \leq n+m$ edges collected as $\{e_{i_1}, e_{i_2}, \dots, e_{i_p}\}$ where we have that each of them has a non-empty label, i.e., $e_{i_k} \in E$, and $\lambda(e_{i_k}) \neq \emptyset$ for every $1 \leq k \leq p$.
Then let the set of itineraries $I$ be $\{ I_{e_{i_1}}, I_{e_{i_2}}, \dots, I_{e_{i_p}} \} \cup \{ I_{e_{i_1}^+}, I_{e_{i_2}^+}, \dots, I_{e_{i_p}^+} \} \cup \{ I_A, I_B \}$, where
we will give the language accepted by each DFA.
The first $2p$ elements have a language with a single string:
for $1 \leq k \leq p$,
to determine the languages $\mathcal{L}(I_{e_{i_k}})$ and $\mathcal{L}(I_{e_{i_k}^+})$, run a breadth first search (BFS) from $v_0$ on $\mathcal{G}$.
This co-routine will return the shortest path (consisting of specific edges) from $v_0$ to $\srcfunc{e_{i_k}}$. 
This path is the only string accepted by $I_{e_{i_k}}$, and the same
path but with $e_{i_k}$ appended is the only string accepted by $I_{e_{i_k}^+}$.

%
 
Next, itinerary DFA $I_A$ is to be defined so it accepts a string
$e_{i_1}e_{i_2}\dots e_{i_r}$ where $e_{i_k} \in E$ for all $1 \leq k \leq r$
if and only if $\tgt(e_{i_r}) \in F_A$.
Similarly, define DFA $I_B$ so that it accepts a string
$e'_{i_1}e'_{i_2}\dots e'_{i_{q}}$ where $e'_{i_k} \in E$ for all $1 \leq k \leq q$
if and only if $\tgt(e'_{i_{q}}) \in F_B$. Note that we are not asking for the
given NFAs $\mathcal{A}$ and $\mathcal{B}$ to be converted to DFAs\,---\,instead, we
are simply constructing a DFA which recognizes that some \textit{path} of an
accepting string arrives at an accepting state in the NFA. The construction of such a
DFA is simple: 
For $I_A$, define two states $q_0$ and $q_1$, with only $q_1$ accepting. 
Then, define transitions from $q_0$ to $q_1$ and $q_1$ to $q_1$ for
all $e \in E$ such that $\tgt(e)$ is a final state in $\mathcal{A}$. Similarly, define transitions
from $q_0$ to $q_0$ and $q_1$ to $q_0$ for all $e \in E$ such that $\tgt(e)$ is
not a final state in $\mathcal{A}$. 
Doing the same  for $\mathcal{B}$ gives $I_B$.

\item Define $I_D = \left\{ [I_{e_{i_1}}, I_{e_{i_1}^+}]^{\apart}, \dots, [I_{e_{i_p}}, I_{e_{i_p}^+}]^{\apart} \right\}$.

\item Finally, define $I_C = \left\{ (I_A, I_B)^{\together} \right\}$.

\end{enumerate}

Lastly, let $k = |\Sigma|$. 

This three-piece mapping is accomplished in polynomial time since the size of the world graph is $O(1 + |\mathcal{A}| + |\mathcal{B}|)$ and the size of $\mathcal{D}$ (i.e., the number of constraints) is $O(|\mathcal{A}| + |\mathcal{B}|)$.\footnotemark
\footnotetext{Here, $|\cdot|$ gives the number of transitions or states, whichever is greater.}
Since BFS runs in polynomial time on $\mathcal{G}$, all the discrimination requirements need polynomial time to construct and each is of polynomial size.
\gobble{In other words, for each itinerary in a discrimination constraint, its singleton language is of polynomial length (since $1 \leq q < |V|$ if $q$ is the length of the shortest path), thus the DFA used to construct it must also be of polynomial size.}
For the itineraries in the conflation constraints, the DFAs have 2 states and $|E|$ transitions\gobblexor{, which is polynomial in the size of $\mathcal{A}$ and $\mathcal{B}$.}{.}

Finally, to prove correctness: there must be a satisfying sensor selection of size at most $k$ if and only if $\mathcal{L}(\mathcal{A}) \subseteq \mathcal{L}(\mathcal{B})$. 

($\implies$) Assume that $\mathcal{L}(\mathcal{A}) \subseteq \mathcal{L}(\mathcal{B})$. Then the sensor selection $M = S$ is a satisfying sensor selection because, firstly, $|M| = |\Sigma| = k$. Secondly, note that all the discrimination constraints are satisfied because all the sensors are turned on. Lastly, the conflation constraint is also satisfied by reasoning as follows: any walk beginning at $v_0$ first going to $q_0^A$ and ending at some $v \in F_A$ has a signature $\{ \sigma_1 \}\{ \sigma_2 \}\dots \{ \sigma_m \}$ for which $\sigma_1\sigma_2\dots \sigma_m \in \mathcal{L}(\mathcal{A})$ which implies $\sigma_1\sigma_2\dots \sigma_m \in \mathcal{L}(\mathcal{B})$. But, by construction, one can take a path in the world graph, taking a first step from $v_0$ to $q_0^B$ without producing any sensor value, and then follow exactly the same path that is accepting in $\mathcal{B}$ through the world graph, and this path will produce signature $\{ \sigma_1 \}\{ \sigma_2 \}\dots \{ \sigma_m \}$.

($\impliedby$) Assume there exists some satisfying sensor selection of size less than or equal to $k = |\Sigma|$. Firstly, no sensor may be turned off since doing so would violate the discrimination constraint between the singleton itineraries involving the edge(s) labelled with the disabled sensor's value.
Thus, the sensor selection has size exactly $k$. Secondly, the conflation constraint is also met implying that, for all signatures $\{ \sigma_1 \}\{ \sigma_2 \}\dots \{ \sigma_m \}$ produced by taking $v_0$ to $q_0^A$ and ending at some $v_i \in F_A$, there exists a path from $v_0$ to $q_0^B$ ending at $v_j \in F_B$ such that its signature is also $\{ \sigma_1 \}\{ \sigma_2 \}\dots \{ \sigma_m \}$. Since no sensor is turned off, the paths that obtain the signatures in the world graph can be taken in $\mathcal{A}$ and $\mathcal{B}$ as well, so $\sigma_1{\sigma_2}\dots \sigma_m \in \mathcal{L}(\mathcal{A})$ implies $\sigma_1{\sigma_2}\dots \sigma_m \in \mathcal{L}(\mathcal{B})$, thus $\mathcal{L}(\mathcal{A}) \subseteq \mathcal{L}(\mathcal{B})$.
\end{proof}

\begin{theorem}
    MSSADDI is \PSPACE-Complete.
\end{theorem}

\begin{proof}
Follows from Lemmas~\ref{MSSADDI-PSPACE} and~\ref{MSSADDI-PSPACE-Hard}.
\end{proof}

\section{Algorithm Description}
\label{algorithm-description}

Having proved the theoretical complexity class of MSSADDI, we now turn to a
description of the algorithm we used to solve it.
Although the algorithm is not polynomial time (as, assuming $\P \neq
\PSPACE$, it couldn't be) we introduce several optimizations to help ameliorate
its running time.

\subsection{Baseline Algorithm}

The approach we chose for solving MSSADDI was a complete enumeration of
subsets, with some shortcutting.  The pseudo-code, based
directly on the automata theoretic connections identified in the preceding,
appears in Algorithm~\ref{alg-mssaddi}. 

It is a top down search over all subsets of $S$ where we attempt to check each
constraint by constructing its signature automaton and verifying the
intersection and subset properties, lines~\ref{disc:inter-check} and~\ref{conf:sub-check}, respectively, as in the previous sections.
Discrimination constraints are checked first (lines~\ref{discr:start}--\ref{discr:end}) because we expect them
to be easier to check than conflation constraints (Lemmas~\ref{nfa-intersection} and~\ref{nfa-inclusion}).

We take advantage of one more property of sensor sets in relation to
discrimination constraints to define our baseline algorithm. Since we stipulate
that different sensors produce different sensor outputs, it follows that if $M
\subseteq S$ does not satisfy a discrimination constraint, then neither can any
subset of $M$. Therefore, when no combination of sensors of size $k$ satisfies
\emph{all} the discrimination constraints, the search is ended, and the current
best satisfying sensor set returned (line~\ref{prev-k-return}).

Next, we propose two optimizations over the baseline algorithm just described.
While each does involve a different trade-off, neither sacrifices the correctness guarantee.

\begin{algorithm}[t!]
\caption{Complete Enumeration for MSSADDI}
\label{alg-mssaddi}

\begin{algorithmic}
\State \textbf{Inputs}: A world graph $\mathcal{G} = (V, E, \src, \tgt, v_0, S,\mathbb{Y},\lambda)$ and a discernment designation $\mathcal{D}=(I, I_D, I_C)$

\State \textbf{Output}: The minimum satisfying  sensor selection, if it exists, otherwise \textbf{null} 

\end{algorithmic}

\begin{algorithmic}[1]

\State $M^* \leftarrow \textbf{null}$
\Comment{The current best sensor set}

\For{$k=|S|$ \textbf{down to} $0$} \label{count-down} 


\For{$M$ in \textsc{Combinations}$(S, k)$} \label{combinations}

\For{$[\mathcal{I}^1, \mathcal{I}^2]^{\apart} \in I_D$} \label{discr:start}
\State $\mathcal{S}_{\mathcal{G}, \mathcal{I}^1, M} \leftarrow \textsc{SignatureAutomaton}(\mathcal{G}, \mathcal{I}^1, M)$  \label{disc:sig1}
\State $\mathcal{S}_{\mathcal{G}, \mathcal{I}^2, M} \leftarrow \textsc{SignatureAutomaton}(\mathcal{G}, \mathcal{I}^2, M)$ \label{disc:sig2}

\If{$\mathcal{L}(\mathcal{S}_{\mathcal{G}, \mathcal{I}^1, M}) \cap \mathcal{L}(\mathcal{S}_{\mathcal{G}, \mathcal{I}^2, M}) \neq \emptyset$} \label{disc:inter-check}
\State \textbf{Continue} to next $M$
\Comment{Check next combination}
\EndIf
\EndFor \label{discr:end}

\For{$(\mathcal{I}_1, \mathcal{I}_2)^{\together}\in I_C$} \label{comb:start}
\State $\mathcal{S}_{\mathcal{G}, \mathcal{I}^1, M} \leftarrow \textsc{SignatureAutomaton}(\mathcal{G}, \mathcal{I}^1, M)$ \label{conf:sig1}
\State $\mathcal{S}_{\mathcal{G}, \mathcal{I}^2, M} \leftarrow \textsc{SignatureAutomaton}(\mathcal{G}, \mathcal{I}^2, M)$ \label{conf:sig2}

\If{$\mathcal{L}(\mathcal{S}_{\mathcal{G}, \mathcal{I}^1, M}) \not\subseteq \mathcal{L}(\mathcal{S}_{\mathcal{G}, \mathcal{I}^2, M})$} \label{conf:sub-check}
\State \textbf{Continue} to next $M$
\Comment{Check next combination}
\EndIf
\EndFor \label{comb:end}

\If{All $I_D$ and $I_C$ satisfied}
\State $M^* \leftarrow M$
\State \textbf{Continue} to next $k$
\Comment{Now try sets of size $k-1$}
\EndIf

\EndFor

\If{No $M$ where $|M| = k$ satisfies all $I_D$}
\State \Return $M^*$ 
\Comment{Prior solution was smallest feasible one} \label{prev-k-return}
\EndIf

\EndFor

\State \Return $M^*$ \label{return-stmt}
\Comment{Final exit}

\end{algorithmic}
\end{algorithm}

\subsection{The Caching Optimization}

Notice how the signature automaton is constructed each time an itinerary is encountered in a constraint (lines \ref{disc:sig1}--\ref{disc:sig2} and \ref{conf:sig1}--\ref{conf:sig2}).
This seems to be wasteful if an itinerary appears in multiple constraints (as it can be with several).
The signature automaton can be cached after it is constructed should the same itinerary appear in another constraint, allowing it to be retrieved without the need for additional computation.

Note, however, the trade-off being made here: while the running time reduced, the space requirements increased. 
Typical library implementations allow for language intersection and subset properties to be checked only on DFA's which, when converted, can result in an exponential increase in space requirements.

\subsection{The Adaptive Weights Optimization}

The second optimization  we introduce is a dynamic reordering of constraints.
Inspired by classical methods in AI for constraint satisfaction problems (CSP's) which seek to make the current assignment \emph{fail fast}, we devised an adaptive weighting mechanism for the desired discernment graph.

Seeking to end the search as fast as possible, discrimination constraints are checked first in the hopes that if none of the sensor sets of cardinality $k$ satisfies the discrimination constraints, then the search can be declared hopeless and ended immediately.
Once a satisfying sensor set is found for the discrimination constraints, though, the following strategy is used. 
Whenever a particular constraint fails to be satisfied, that sensor set `votes' the erring constraint up so that future sets know which constraint is likely to fail.
Thus, after a few iterations, enough data is collected so that a sensor set checks that constraint first which most of the sets before it failed on.
The idea is the more demanding (or stringent) constraints are learned and propagated upward for prioritization.

\section{Experimental results}
\label{experimental-results}

The following experiments were all performed on a computer running Windows~11 with an Intel i7 CPU having \num{16} GB RAM using Python 3.

As a basic sanity check, we ran the baseline algorithm on the problems presented in Section~\ref{introduction-section}. For these problems, the algorithm correctly provided the optimal solutions in less than \SI{1}{\second}.
Next, to test the scalability of the proposed approach and to assess the impact of the optimizations, we ran the experiments that follow.


\subsection{Test cases}

\begin{figure}[t]
  \centering
  \includegraphics[width=0.68\linewidth]{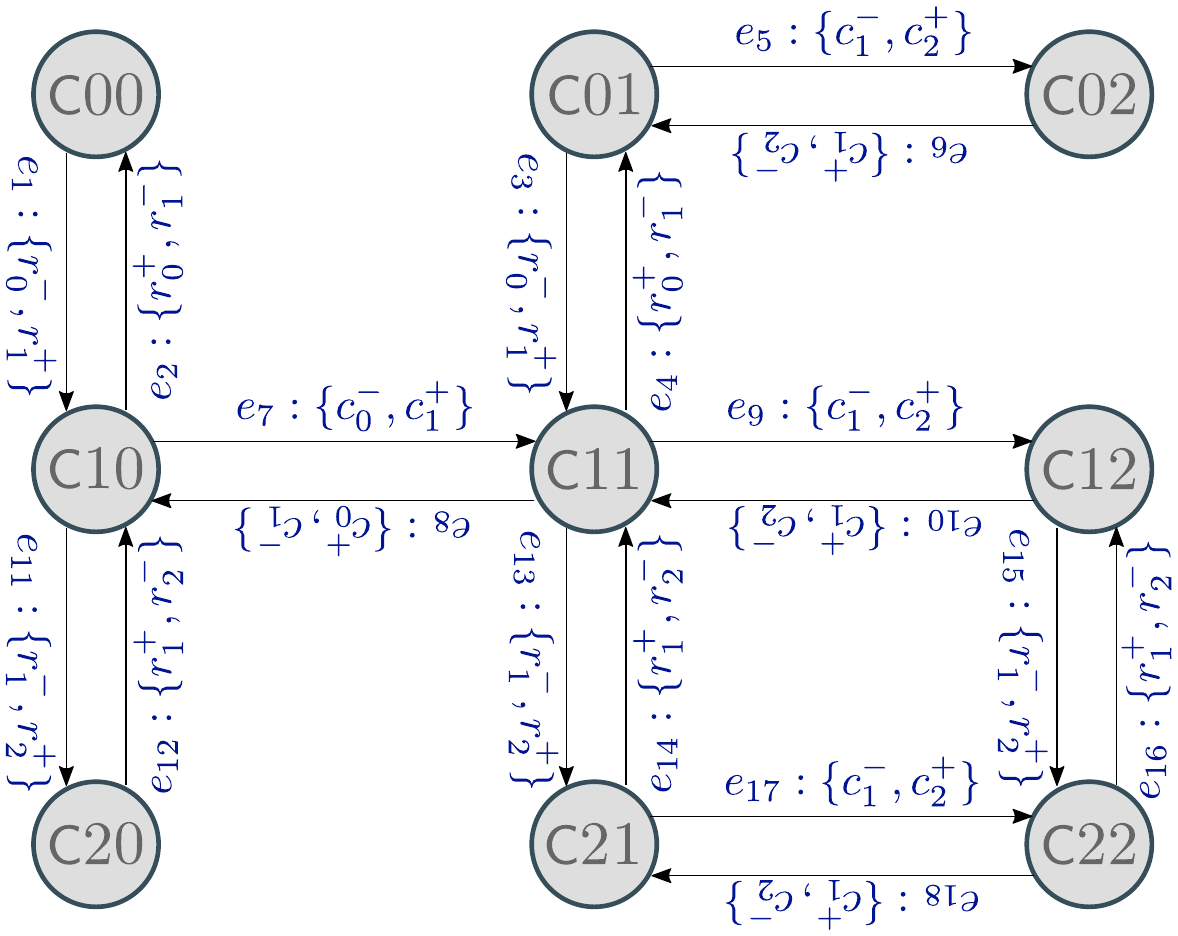}
  \begin{minipage}{0.38\linewidth}
  \phantom{ww}~\\
  \includegraphics[width=0.97\linewidth]{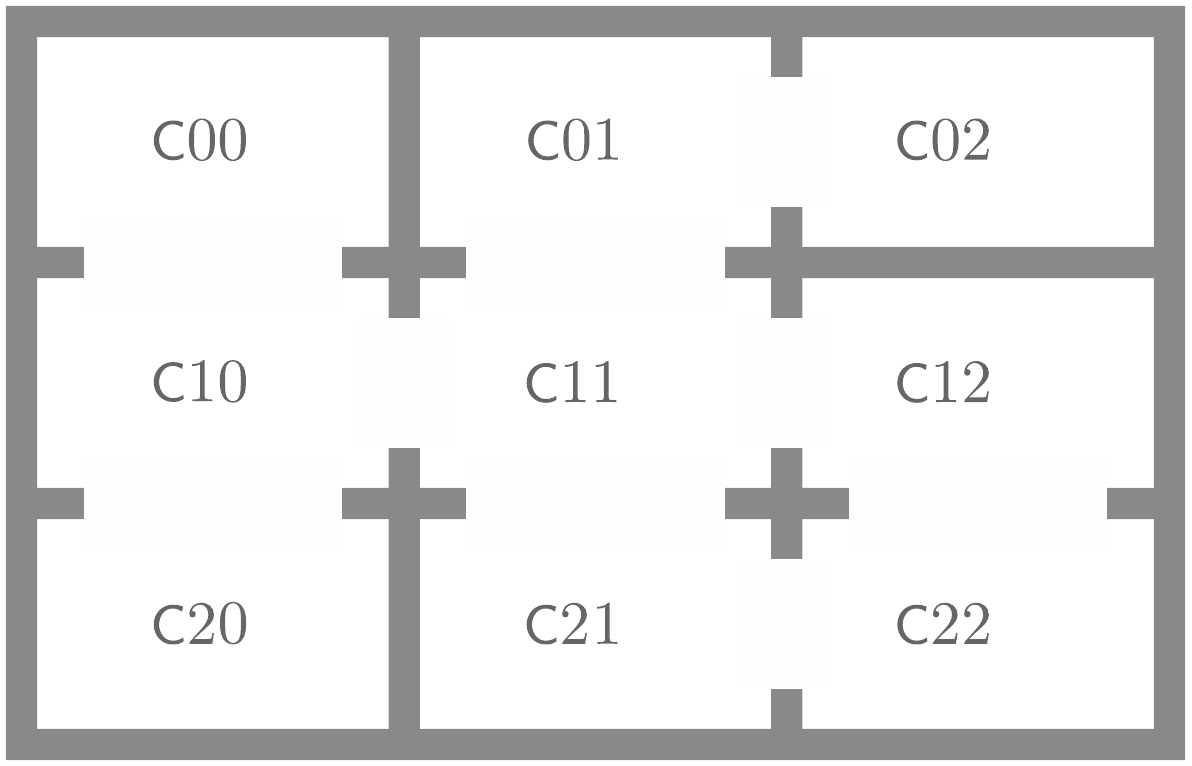}
  \end{minipage}
  \begin{minipage}{0.58\linewidth}
  \caption{
  An example world for testing the algorithm's scaling. This is the $3 \times 3$ case, with blueprint of the map at left, and labelled world graph shown above. There are \num{3} row and \num{3} column sensors available which detect the agent's presence.
    The problem is choose a subset of these.
  }
  \label{3x3Graph}
  \end{minipage}
  \vspace*{-12pt}
\end{figure}

The test cases we propose are designed such that they are parameterized:
we use an $m \times n$ grid-type world graph. 
An example with $m=n=3$ is shown in Figure \ref{3x3Graph}, with the scaled versions 
adding cells rightward and downward (without any missing edges unlike the figure).
There is a sensor in each row that registers the fact that agent is present within the associated row. Similarly, a column sensor detects when the agent is within that column. 
Sensor set $S$ consists of \mbox{$m+n$} sensors, one for each row and each column.
The figure shows the labelled world graph, this small instance with
\num{18} edges, the arcs each bearing their $\lambda$-based labelling. These
follow a simple pattern: for example, $r_2^+$  means that row $2$'s sensor has
triggered, going from the unoccupied to occupied state; while $c_1^-$ means
that column $1$'s sensor has gone from the occupied to unoccupied.

Finally, we construct an itinerary for every state in the world graph where the language accepted by the DFA for the itinerary describes following any edge in the world graph any number of times followed by an edge incoming to this state.
Essentially, the itinerary DFA for that state accepts a string of edges if and only if the last edge that was taken in that walk was an incoming edge to that state.

The number of constraints are proportional to the number of states in the world graph.
We add $mn$ discrimination constraints each by randomly selecting any 2 itineraries which describe ending in two states which are in a different column \emph{and} in a different row.
Similarly, we also add $m$ conflation constraints per column, each between 2 random itineraries that describe ending in different rows in that column. 
Thus, in expectation, each itinerary is in 2 discrimination constraints and 2 conflation constraints.


\subsection{Solutions}

From the description of the problem above, it should be clear that activating either only the row sensors or only the column sensors should be a satisfying sensor selection for the discrimination constraints alone. 
After all, ending in a different row and column can be distinguished on the basis of either information provided by a row sensor or a column sensor.
However, when considering both the discrimination and conflation constraints, only one of these options becomes feasible\,---\,namely, that involving only activating the column occupancy sensors.
Activating a row sensor could potentially violate some conflation constraints which describe ending in that row.
Note that we see another detail of MSSADDI reiterated here\,---\,that when $n > m$, it may be necessary to activate more sensors (i.e., column sensors as opposed to only the row sensors) to satisfy the both upper and lower information bounds as opposed to the lower bounds alone.

\subsection{Analysis}

\begin{figure}[t!]
  \centering
\gobblexor{\hspace*{-4pt}}{\hspace*{-8pt}}\includegraphics[width=1.02\linewidth]{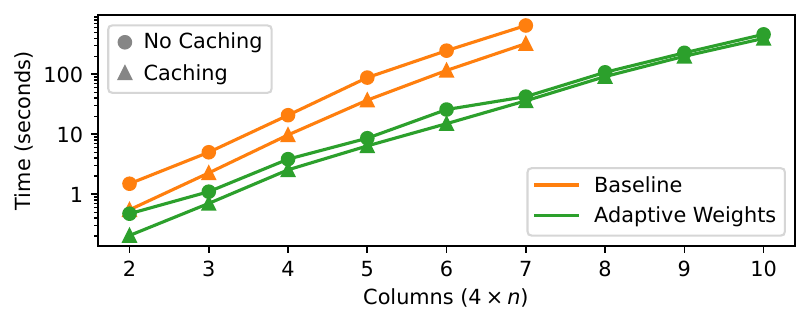}
  \vspace*{-20pt}
  \caption{
  Effect of all the optimizations for various grid sizes.
  }
  \label{AllOptimizationsGraph}
  \vspace*{-10pt}
\end{figure}

The basic scaling plot for various grid sizes is shown in Figure \ref{AllOptimizationsGraph}.
As can be seen in that plot, using the caching optimization alone led on average to a \SI{53.5}{\percent} reduction in the running time.
For our purposes, all the signature automata were able to be cached, and memory did not seem to be an issue (i.e., we never received an out-of-memory exception).
Thus, time, not space, seemed to be a dominating factor in solving this problem with current resources.

The results are even more impressive for the adaptive weights optimization.
As compared to the baseline algorithm, it led on average to a \SI{87.6}{\percent} improvement in running time. 
When both optimizations are applied together, however, caching the signature automata seems to have little effect when adaptive weights are already in use.
This makes sense because the adaptive weights allow a sensor set to be determined as unsatisfiable fast, lowering the probability that the same itinerary will be checked more than once.

\begin{figure}[t!]
  \centering
\gobblexor{\hspace*{-4pt}}{\hspace*{-8pt}}\includegraphics[width=1.04\linewidth]{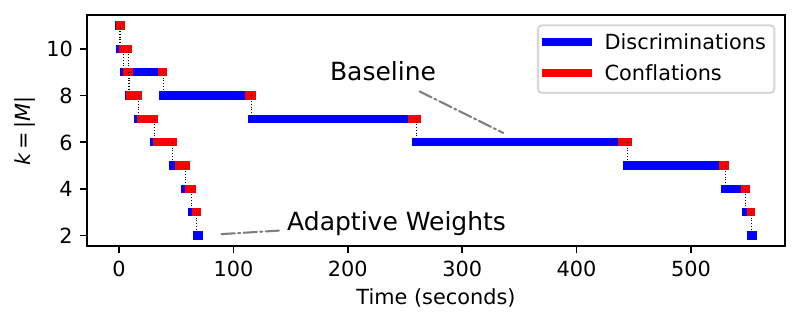}
  \vspace*{-16pt}
  \caption{
  Effect of dynamically reordering constraints when checking each sensor combination. The horizontal axis shows the progression of time and the vertical axis the size of the sensor set being checked.
  }
  \vspace*{-13pt}
  \label{SteppingGraph65}
\end{figure}

Seeking to understand how the mix of constraints checked changes when adaptive weights are used, we decided to analyze the time spent by the algorithm in different parts of the code for the $6\times 5$ world graph grid.
We measured the wall clock every time the algorithm started checking subsets of size $k$ (see line~\ref{count-down} in Algorithm \ref{alg-mssaddi}). 
Furthermore, we also kept count of the number of discrimination and conflation constraints checked for each sensor set aggregated over size $k$ before it failed.
The results, including a visualization of the constraint type, appear in the stepping chart in Figure~\ref{SteppingGraph65}.

Notice, first, how the optimization leads to a greater proportion of conflation constraints being checked. 
For our case, conflation constraints tend to fail more often when the sensor set is of high cardinality since they are likely to include row sensors. 
Thus, a greater proportion (or sometimes even absolutely more) of them are checked, as compared to baseline.
We see that the decision, on the basis of Lemmas~\ref{nfa-intersection} and~\ref{nfa-inclusion}, to place lines~\ref{discr:start}--\ref{discr:end} before lines~\ref{comb:start}--\ref{comb:end} may be mistaken, on average.

Secondly, observe how the algorithm is able to terminate after concluding that no set of size $k=2$ will satisfy all the discrimination constraints.
The minimum satisfying sensor set in this case turned out to be \num{3} column sensors.

\section{Conclusion and Future Works}

This paper tackled the sensor selection problem for multiple itineraries while also allowing for considerations of privacy.
We also provided strong reasoning for why merely minimizing selected sensors does not lead to satisfaction of specific privacy requirements.
We formulated this problem  and proved that it was worst-case intractable.
Further, we provided an algorithm (based on automata-theoretic operations)
to solve the problem and considered a few optimizations over the na\"\i ve
implementation.
In the process, we realized that the gains from those optimizations were significant owing to an inclination for wanting incorrect solutions to fail fast.

In the future, research might seek a direct reduction from the problem we proposed to canonical \PSPACE-Complete problems such as QSAT.
Other approaches common to solving computationally hard problems such as random algorithms, and improved heuristics may also be fruitful.

\subsection{Acknowledgements}

This material is based upon work supported in part by the 
National Science Foundation under 
grant~\href{http://nsf.gov/awardsearch/showAward?AWD_ID=2034097}{IIS-2034097}
and DoD Army Research Office under
award~W911NF2120064.

\bibliographystyle{plain} 
\bibliography{refs} 

\begin{thebibliography}{10}

\bibitem{cassez2012synthesis}
Franck Cassez, J{\'e}r{\'e}my Dubreil, and Herv{\'e} Marchand.
\newblock Synthesis of opaque systems with static and dynamic masks.
\newblock {\em Formal Methods in System Design}, 40:88--115, 2012.

\bibitem{cassez2008fault}
Franck Cassez and Stavros Tripakis.
\newblock Fault diagnosis with static and dynamic observers.
\newblock {\em Fundamenta Informaticae}, 88(4):497--540, 2008.

\bibitem{chou2013security}
Te-Shun Chou.
\newblock Security threats on cloud computing vulnerabilities.
\newblock {\em International Journal of Computer Science \& Information
  Technology}, 5(3):79, 2013.

\bibitem{cortes2016differential}
Jorge Cort{\'e}s, Geir~E Dullerud, Shuo Han, Jerome Le~Ny, Sayan Mitra, and
  George~J Pappas.
\newblock Differential privacy in control and network systems.
\newblock In {\em IEEE 55th Conference on Decision and Control (CDC)}, pages
  4252--4272, 2016.

\bibitem{hopcroft2006introduction}
John~E Hopcroft, Rajeev Motwani, and Jeffrey~D Ullman.
\newblock {\em Introduction to Automata Theory, Languages, and Computation}.
\newblock Addison-Wesley, 3 ed., 2006.

\bibitem{jacob2016overview}
Romain Jacob, Jean-Jacques Lesage, and Jean-Marc Faure.
\newblock Overview of discrete event systems opacity: Models, validation, and
  quantification.
\newblock {\em Annual reviews in control}, 41:135--146, 2016.

\bibitem{lafortune2018history}
St{\'e}phane Lafortune, Feng Lin, and Christoforos~N Hadjicostis.
\newblock On the history of diagnosability and opacity in discrete event
  systems.
\newblock {\em Annual Reviews in Control}, 45:257--266, 2018.

\bibitem{lin2011opacity}
Feng Lin.
\newblock Opacity of discrete event systems and its applications.
\newblock {\em Automatica}, 47(3):496--503, 2011.

\bibitem{masopust2019complexity}
Tom{\'a}{\v{s}} Masopust and Xiang Yin.
\newblock Complexity of detectability, opacity and a-diagnosability for modular
  discrete event systems.
\newblock {\em Automatica}, 101:290--295, 2019.

\bibitem{meyer1972equivalence}
Albert~R Meyer and Larry~J Stockmeyer.
\newblock The equivalence problem for regular expressions with squaring
  requires exponential space.
\newblock In {\em SWAT}, volume~72, pages 125--129, 1972.

\bibitem{okane15discreet}
Jason~M. O'Kane and Dylan~A. Shell.
\newblock {Automatic Design of Discreet Discrete Filters}.
\newblock In {\em IEEE International Conference on Robots and Automation
  (ICRA)}, Seattle, WA, May 2015.

\bibitem{prorok17privacy}
A.~Prorok and V.~Kumar.
\newblock {Privacy-Preserving Vehicle Assignment for Mobility-on-Demand
  Systems}.
\newblock In {\em IEEE/RSJ International Conference on Intelligent Robots and
  Systems (IROS)}, Vancouver, Canada, 2017.

\bibitem{rahmani2021sensor}
Hazhar Rahmani, Dylan~A. Shell, and Jason~M. O'Kane.
\newblock {Sensor selection for detecting deviations from a planned itinerary}.
\newblock In {\em IEEE/RSJ International Conference on Intelligent Robots and
  Systems (IROS)}, pages 6511--6518, Online, 2021.

\bibitem{ramos2021daily}
Ra{\'u}l~G{\'o}mez Ramos, Jaime~Duque Domingo, Eduardo Zalama, and Jaime
  G{\'o}mez-Garc{\'\i}a-Bermejo.
\newblock Daily human activity recognition using non-intrusive sensors.
\newblock {\em Sensors}, 21(16):5270, 2021.

\bibitem{rowe2005detecting}
Neil~C. Rowe.
\newblock Detecting suspicious behavior from positional information.
\newblock In {\em Modelling Others from Observations Workshop at IJCAI.
  Edinburgh, Scotland}, 2005.

\bibitem{rueben2017taxonomy}
Matthew Rueben, Cindy~M Grimm, Frank~J Bernieri, and William~D Smart.
\newblock A taxonomy of privacy constructs for privacy-sensitive robotics.
\newblock {\em arXiv preprint arXiv:1701.00841}, 2017.

\bibitem{sampath1995diagnosability}
Meera Sampath, Raja Sengupta, St{\'e}phane Lafortune, Kasim Sinnamohideen, and
  Demosthenis Teneketzis.
\newblock Diagnosability of discrete-event systems.
\newblock {\em {IEEE Trans. on Auto. Control}}, 40(9):1555--1575, 1995.

\bibitem{savitch1970relationships}
Walter~J Savitch.
\newblock Relationships between nondeterministic and deterministic tape
  complexities.
\newblock {\em Journal of Computer and System Sciences}, 4(2):177--192, 1970.

\bibitem{uddin2018ambient}
Md~Zia Uddin, Weria Khaksar, and Jim Torresen.
\newblock Ambient sensors for elderly care and independent living: a survey.
\newblock {\em Sensors}, 18(7):2027, 2018.

\bibitem{wang18}
Rui Wang, Yingxia Wei, Houbing Song, Yu~Jiang, Yong Guan, Xiaoyu Song, and
  Xiaojuan Li.
\newblock {From Offline Towards Real-Time Verification for Robot Systems}.
\newblock {\em IEEE Trans. on Industrial Informatics}, 14(4):1712--1721, 2018.

\bibitem{wang2018optimizing}
Weilin Wang, Chaohui Gong, and Di~Wang.
\newblock Optimizing sensor activation in a language domain for fault
  diagnosis.
\newblock {\em {IEEE Trans. on Auto. Control}}, 64(2):743--750, 2018.

\bibitem{yin2017minimization}
Xiang Yin and St{\'e}phane Lafortune.
\newblock Minimization of sensor activation in decentralized discrete-event
  systems.
\newblock {\em {IEEE Trans. on Auto. Control}}, 63(11):3705--3718, 2017.

\bibitem{yu2010cyber}
Jingjin Yu and Steven~M LaValle.
\newblock Cyber detectives: Determining when robots or people misbehave.
\newblock In {\em Algorithmic Foundations of Robotics IX}, pages 391--407.
  2010.

\bibitem{yu2011story}
Jingjin Yu and Steven~M LaValle.
\newblock Story validation and approximate path inference with a sparse network
  of heterogeneous sensors.
\newblock In {\em IEEE International Conference on Robotics and Automation
  (ICRA)}, pages 4980--4985, 2011.

\bibitem{zaytoon2013overview}
Janan Zaytoon and St{\'e}phane Lafortune.
\newblock Overview of fault diagnosis methods for discrete event systems.
\newblock {\em Annual Reviews in Control}, 37(2):308--320, 2013.

\end{thebibliography}

\end{document}